\documentclass[10pt]{article} % For LaTeX2e
\usepackage[accepted]{rlc}
\usepackage{amssymb}            % Defines common symbols like \mathbb R
\usepackage{amsthm} % Defines proof
\usepackage{mathtools}          % Extends amsmath, providing common math tools
\usepackage{mathrsfs}           % Enables \mathscr, which can work in cases that \mathcal does not
\mathtoolsset{showonlyrefs}     % Only number equations that are referenced (optional)
\usepackage{graphicx}           % For including images
\usepackage{subcaption}         % Allows for the use of subfigures and subcaptions
\usepackage[space]{grffile}     % For spaces in image names
\usepackage{url}                % For displaying urls
\usepackage{algorithm}
\usepackage{algorithmic}
\usepackage{enumitem}
\usepackage{dsfont}
\usepackage[normalem]{ulem}

\newtheorem{theorem}{Theorem}[section]
\newtheorem{proposition}[theorem]{Proposition}

\newtheorem{corollary}[theorem]{Corollary}

\DeclareMathOperator*{\EV}{\mathbb{E}}

\renewcommand{\S}{\mathcal{S}}
\newcommand{\X}{\mathcal{X}}
\newcommand{\A}{\mathcal{A}}

\renewcommand{\P}{\mathbb{P}}
\newcommand{\Obs}{\mathbb{O}}

\newcommand{\B}{\mathcal{B}}

\newcommand{\V}{\mathbb{V}}
\newcommand{\norm}[1]{\left\lVert#1\right\rVert}

\DeclareMathOperator*{\vs}{\mathbf{s}}
\DeclareMathOperator*{\vx}{\mathbf{x}}
\DeclareMathOperator*{\va}{\mathbf{a}}
\DeclareMathOperator*{\vH}{\mathbf{H}}

\newcommand\margincomment[4]

\title{The Limits of Pure Exploration in POMDPs: \;
When the Observation Entropy is Enough}

\author{Riccardo Zamboni  \\
    riccardo.zamboni@polimi.it \\
    Politecnico di Milano
    \And
    Duilio Cirino  \\
    duilio.cirino@mail.polimi.it \\
    Politecnico di Milano
    \And
    Marcello Restelli  \\
    marcello.restelli@polimi.it \\
    Politecnico di Milano
    \And
    Mirco Mutti  \\
    mirco.m@technion.ac.il \\
    Technion}

\begin{document}

\maketitle

\begin{abstract}
The problem of \emph{pure exploration} in Markov decision processes has been cast as maximizing the entropy over the state distribution induced by the agent's policy, an objective that has been extensively studied.
However, little attention has been dedicated to state entropy maximization under \emph{partial observability}, despite the latter being ubiquitous in applications, e.g., finance and robotics, in which the agent only receives noisy observations of the true state governing the system's dynamics.
How can we address state entropy maximization in those domains? In this paper, we study the simple approach of maximizing the \emph{entropy over observations} in place of true latent states. First, we provide lower and upper bounds to the approximation of the true state entropy that only depend on some properties of the observation function. Then, we show how knowledge of the latter can be exploited to compute a principled regularization of the observation entropy to improve performance. With this work, we provide both a flexible approach to bring advances in state entropy maximization to the POMDP setting and a theoretical characterization of its intrinsic limits.
\end{abstract}

\allowdisplaybreaks

%%%%%%%%%%%%%%%%%%%%%%%%%%%%%%%%%%%%%%%%%%%%%%%%%%%%%%%%%%%%%%%%
%% Section: INTRODUCTION
%%%%%%%%%%%%%%%%%%%%%%%%%%%%%%%%%%%%%%%%%%%%%%%%%%%%%%%%%%%%%%%%
\section{Introduction}
\label{sec:introduction}

% \rz{aggiungere reference a ICML? serve? Questo mi sembra più autocontenuto, forse lo metterei al massimo nei related work ma non in introduzione}

A plethora of recent works~\citep{hazan2019provably, lee2019smm, mutti2020ideal, tarbouriech2020active, zhang2020exploration, guo2021geometric, liu2021behavior, liu2021aps, seo2021state, yarats2021reinforcement, mutti2021task, mutti2022importance, mutti2022unsupervised, nedergaard2022k, yang2023cem, tiapkin2023fast, jain2023maximum, kim2023accelerating, zisselman2023explore, mutti2023unsupervised} have studied \emph{state entropy maximization} for pure exploration of Markov Decision Processes~\citep[MDPs,][]{puterman2014markov} in the absence of a reward function. 
In this Maximum State Entropy (MSE) framework, formally introduced by~\citet{hazan2019provably}, the agent aims to maximize the entropy of the state visitation induced by its policy instead of the cumulative sum of rewards, which is a generalization of the Reinforcement Learning~\citep[RL,][]{bertsekas2019reinforcement} problem that often goes as \emph{convex} RL~\citep{zahavy2021reward, geist2021concave, mutti2022challenging, mutti2023convex} due to the convexity of the entropy function.

Despite the problem being harder than RL, MSE has brought remarkable empirical success as a tool for data collection~\citep{yarats2022don}, transition model estimation~\citep{tarbouriech2020active}, and policy pre-training~\citep{mutti2021task}, as well as an essential building block for improved skills discovery~\citep{liu2021aps} and generalization across various tasks~\citep{zisselman2023explore}.

Especially, \citet{mutti2021task, liu2021behavior, seo2021state, yarats2021reinforcement} have popularized a practical \emph{policy optimization} procedure that allows to address MSE at scale. Their method is based on pairing flexible nearest-neighbors estimators of the entropy~\citep{singh2003nearest} with policies implemented through neural networks trained via backpropagation, a recipe for success in complex and high-dimensional domains, e.g., continuous control or learning from images.

Although many facets of the MSE problem have been studied, all of the previous works assume the states, on which the entropy is maximized, to be fully observable to the agent. However, this is not the case for several interesting applications. Let us think of a trading scenario: A trader typically accesses a small portion of the true state, e.g., current stock prices, volumes, and so on, while other parts, e.g., the general sentiment of the market or companies' revenues published quarterly, mostly remain latent, albeit crucial to define the system's dynamics. The same goes for a robotic navigation task, where the state is often accessed through noisy sensory inputs, such as cameras and proximity, rather than true spatial coordinates. 
%\textcolor{red}{In such a setting though, the reward function is usually defined over the latent state space, and the ability to explore states uniformly gains a particularly crucial role in terms of policy pre-training for skills discovery and task generalization.}
An important question arises naturally: 
\begin{center} \vspace{-2pt}
    \emph{Can we maximize the entropy over states getting partial observations only?}
\end{center} \vspace{-2pt}
While the problem of addressing a learning objective that is hidden from the agent is fascinating per se, we argue that any improvement in this direction is paramount to bringing MSE closer to practical applications. Unfortunately, the problem we just described is a clear generalization of learning in Partially Observable MDPs~\citep[POMDPs,][]{aastrom1965optimal}, which is well-known to be intractable.

In this paper, we study the simple approach of maximizing the entropy over the partial observations in place of the true states. This framework, which we call \emph{Maximum Observation Entropy} (MOE), gives two crucial benefits. On the one hand, we can sidestep the inherent computational hardness of dealing with POMDPs~\citep{papadimitriou1987complexity}, as the class of policies that are Markovian over observations suffices~\citep{hazan2019provably} and we do not need to build complex belief distributions over the true state as the objective is fully observable. Secondly, all of the previous implementations can be directly transferred from MSE to MOE without changes, which gives a head start to the MOE problem instead of implementing new techniques from the ground up.

Surprisingly, we can show that this straightforward approach is not hopeless. We derive formal approximation results on the difference between the entropy induced over observations and the corresponding entropy over the true (latent) states, which can be upper bounded as a function of properties of the observation matrix, namely its maximum singular value or the average entropy of its rows. This is in stark contrast with RL in POMDPs, in which the optimal policy over observations is almost arbitrarily sub-optimal under similar assumptions.

Whereas the approximation bounds characterize settings where optimizing for MOE is enough, they tell little about how to address domains in which the observation matrix is not so well-behaved.
In those settings, we show that knowledge of the observation matrix, a reasonable requirement in some applications (e.g., the specifics of sensors and cameras equipped on a robot may be available), can be exploited to improve performance over MOE. First, we derive a principled regularization term that discounts the entropy induced by the observations with the entropy of their emission, intuitively putting more weight on the observations for which the emission process is reliable and less on those that are known to be noisy. Then, we incorporate the latter regularization in an appropriate policy gradient algorithm inspired by previous MSE approaches~\citep{mutti2021task, liu2021behavior}. Finally, we test the algorithm on a set of simple yet illustrative domains to validate our theoretical findings, bringing an algorithmic blueprint for scalable state entropy maximization in POMDPs.

\paragraph{Contributions.} Throughout the paper, we make the following contributions:
\begin{itemize}[itemsep=-1pt, leftmargin=*, topsep=2pt]
    \item In Section~\ref{sec:problem_formulation}, we provide the first generalization of the MSE problem to the POMDP setting, introducing MOE as a simple yet flexible and tractable approach;
    \item In Section~\ref{sec:characterization}, we theoretically analyze the gap between MOE and MSE, providing a family of upper and lower bounds that link the approximation error to spectral and information properties of the observation matrix;
    \item In Section~\ref{sec:algorithm}, we design a policy gradient algorithm for (general) MOE optimization, and we provide a variation including a principle regularization term to exploit knowledge of the observation matrix (but not the POMDP specification);
    \item In Section~\ref{sec:experiments}, we report an empirical validation that tests the introduced algorithms against an ideal baseline accessing the true states of the POMDP, which both upholds the algorithms' design and our theoretical findings.
\end{itemize}
Finally, in a concurrent work~\citep{zamboni2024explore} we explore state entropy maximization in POMDPs beyond MOE, studying methodologies that exploit observations to build \emph{beliefs} over the true states and then maximize the entropy of the \emph{believed} states.
With the combined contributions of this paper and~\citet{zamboni2024explore}, we hope to pave the way for future studies of state entropy maximization in partially observable environments.

%%%%%%%%%%%%%%%%%%%%%%%%%%%%%%%%%%%%%%%%%%%%%%%%%%%%%%%%%%%%%%%%
%% Section: PRELIMINARIES
%%%%%%%%%%%%%%%%%%%%%%%%%%%%%%%%%%%%%%%%%%%%%%%%%%%%%%%%%%%%%%%%
\section{Preliminaries}
\label{sec:preliminaries}

In this section, we introduce the most relevant background and the basic notation.

\paragraph{Notation.}
In the following, we denote $[N] := \{1, 2, \ldots, N\}$ for a constant $N < \infty$. We denote a set with a calligraphic letter $\A$ and its size as $|\A|$. We denote $\A^T := \times_{t = 1}^T \A$ the $T$-fold Cartesian product of $\A$. The simplex on $\A$ is denoted as $\Delta_{\A} := \{ p \in [0, 1]^{|\A|} | \sum_{a \in \A} p(a) = 1 \}$ and $\Delta_{\A}^{\B}$ denotes the set of conditional distributions $p: \A \to \Delta_\B$. Let $X$ a random variable on the set of outcomes $\X$ and corresponding probability measure $p_X$, we denote as $H_\alpha(X) = \frac{1}{1-\alpha} \log(\sum_{x \in \X} p_X (x)^\alpha)$ the R\`enyi entropy of order $\alpha$, from which we recover the Shannon entropy $H (X) = \lim_{\alpha \to 1} H_\alpha (X) = - \sum_{x \in \X} p_X (x) \log (p_X (x))$ and the min-entropy $H_\infty (X) = \lim_{\alpha \to \infty} H_\alpha (X) = - \log (\max_{x \in \X} p_X (x))$.
We denote $\vx = (X_1, \ldots, X_T)$ a random vector of size $T$ and $\vx[t]$ its entry at position $t \in [T]$.
For a vector $v \in \mathbb{R}^N$ we denote $\| v \|_\infty := \max_{i \in [N]} v_i$. For a matrix $\V \in \mathbb{R}^{N \times M}$ we denote $\|\V\|_\infty := \max_{ij \in [N] \times [M]} V_{ij}$ its infinity norm, $\V^*$ its conjugate transpose and $\V^{\circ - 1}$ its Hadamard inverse, such that $V_{ij}^{\circ - 1} = 1 / V_{ij} \ \forall i,j$. We further denote $\lambda (\V), \sigma (\V)$, the vectors of eigenvalues and singular values of $\V$, respectively. We denote the spectral norm of $\V$ as $\| \V \|_2 := \sqrt{\lambda_{\max} (\V^* \V)} = \sigma_{\max} (\V)$ where $\lambda_{\max} (\V) := \| \lambda (\V) \|_\infty$ and $ \sigma_{\max} (\V) := \| \sigma (\V) \|_\infty$.

As a base model for interaction, we consider a finite-horizon Partially Observable Markov Decision Process~\cite[POMDP,][]{aastrom1965optimal} without rewards. A POMDP $\mathbb{M} := (\S, \X, \A, \Obs, \P, \mu, T)$ is composed of a set $\S$ of latent states, a set $\X$ of observations, and a set of actions $\A$, which we let discrete and finite with size $|\S|, |\X|, |\A|$ respectively. At the start of an episode, the initial state $s_1$ of $\mathbb{M}$ is drawn from an initial state distribution $\mu \in \Delta_{\S}$. An agent interacting with $\mathbb{M}$ never accesses the true state of the system but an observation $x_1 \sim \Obs (\cdot | s_1)$ where $\Obs \in \Delta_{\X}^{\S}$ is the \emph{observation function}.\footnote{With slight overload of notation, we will equivalently represent the observation function through a stochastic matrix $\Obs \in \mathbb{R}^{|\S| \times |\X|}$.} Upon observing $x_1$, the agent takes action $a_1 \in \A$, the system transitions to $s_2 \sim \P(\cdot|s_1, a_1)$ according to the \emph{transition model} $\P \in \Delta_{\S}^{\S \times \A}$, and a new observation $x_2 \sim \Obs (\cdot | s_2)$ is generated. The process is repeated until $s_T$ is reached and $x_T$ is generated, being $T < \infty$ the horizon of an episode.

The agent selects actions according to a decision \emph{policy} $\pi \in \Delta_{\X}^{\A}$ such that $\pi (a | x)$ denotes the conditional probability of taking action $a$ upon observing $x$. Deploying a policy $\pi$ over a POMDP $\mathbb{M}$ induces a specific distribution over states and observations. Let denote as $S, X$ the random variables corresponding to the state and observation respectively, we have that $S$ is distributed as $p^\pi_S \in \Delta_{\S}$, where $p^\pi_S (s) = \frac{1}{T} \sum_{t \in [T]} Pr (s_t = s)$, and $X$ as $p^\pi_X \in \Delta_{\X}$, where $p^\pi_X (x) = \frac{1}{T} \sum_{t \in [T]} Pr (x_t = x)$. Further, it is easy to see that $p^\pi_X (x) = \sum_{s \in \S} p^\pi_S (s) \Obs (x | s)$. Furthermore, let us denote with $\vs, \vx, \va$ the random vectors corresponding to sequences of states, observations, and actions of length $T$, which are supported in $\S^T, \X^T, \A^T$ respectively. We have that $\vs$ is distributed as $q^\pi_S \in \Delta_{\S^T}$, where $q^\pi_S(\textbf{s}) = \prod_{t \in [T]}Pr(s_t = \textbf{s}[t])$, and $\vx,\va$ as $q^\pi_{XA} \in \Delta_{\X^T \times \A^T}$, where $q^\pi_{XA}(\vx,\va) = \prod_{t \in [T]}Pr(x_t = \vx[t], a_t = \va[t])$. Finally, we denote the empirical distributions induced by $\vs, \vx$ as $\hat p_S (s|\vs) = \frac{1}{T} \sum_{t \in [T]}  \mathds{1} (\vs[t] = s)$ and $\hat p_X (x|\textbf{x}) = \frac{1}{T} \sum_{t \in [T]} \mathds{1} (\textbf{x}[t] = x)$, which does not depend on $\pi$ due to the conditioning on $\vs, \vx$.

%%%%%%%%%%%%%%%%%%%%%%%%%%%%%%%%%%%%%%%%%%%%%%%%%%%%%%%%%%%%%%%%
%% Section: PROBLEM
%%%%%%%%%%%%%%%%%%%%%%%%%%%%%%%%%%%%%%%%%%%%%%%%%%%%%%%%%%%%%%%%
\section{Problem Formulation}
\label{sec:problem_formulation}

In the MDP setting, i.e., observations coincide with the true states of the state of the system, \citet{hazan2019provably} have formulated the \emph{Maximum State Entropy} (MSE) objective as follows
\begin{equation}
\label{eq:mse}
    \max_{\pi \in \tilde\Pi} \ \Big\{ H (S | \pi) := - \sum\nolimits_{s \in \S} p^\pi_S (s) \log p^\pi_S (s) \Big\}
\end{equation}
where $\tilde\Pi \subseteq \Delta_{\S}^{\A}$ is the set of Markovian policies from states to distribution over actions, and $H(S | \pi)$ is the entropy of the state variable $S$ ``conditioned'' on running the policy $\pi$ in the MDP. When the MDP is fully known, \citet{hazan2019provably} shows that~\eqref{eq:mse} is non-convex, but it admits a convex dual formulation, which is optimized by a stochastic Markovian policy in general.
% A POMDP such that $\O = \S$ and $\Obs (s | s) = 1$ reduces to a finite-horizon Markov Decision Process~\citep[MDP,][]{puterman2014markov} $\M := (\S, \A, \P, T, \mu)$. In this setting, \citet{mutti2022importance} proposed the single-trajectory \emph{Maximum State Entropy} (MSE) objective
% \begin{equation}
%     \label{eq:mdp_entropy_single_trial}
%     \max_{\pi \in \Pi_{\I}} \ \Big\{J^\S(\pi) := \EV_{S \sim p^\pi} [ H (S)] \Big\}
% \end{equation}
% in which the agent seeks to maximize the entropy of the empirical state distribution induced in one single trajectory rather than in multiple trajectories (as in~\citet{hazan2019provably}). In this case, the agent can directly take actions according to the true state, nonetheless, the optimal policy in~\eqref{eq:mdp_entropy_single_trial} is known to be deterministic non-Markovian in general, which makes the optimization problem computationally hard~\cite{mutti2022importance}.

In principle, we aim to address the same objective~\eqref{eq:mse} in POMDPs as well. However, in the POMDP setting, we cannot access the true states, which are latent, but we have to rely on partial observations generated from those states. Thus, a straightforward adaptation of~\eqref{eq:mse} to POMDPs is to define an analogous objective on observations as a proxy for $H(S|\pi)$, which we cannot access. We define the \emph{Maximum Observation Entropy} (MOE) objective as follows
\begin{equation}
\label{eq:moe}
    \max_{\pi \in \Pi} \ \Big\{ H (X | \pi) := - \sum\nolimits_{x \in \X} p^\pi_X (x) \log p^\pi_X (x) \Big\}
\end{equation}
where $\Pi \subseteq \Delta_{\S}^{\A}$ is the set of Markovian policies from observations to distribution over actions, and $H(X | \pi)$ is the entropy of the observation $X$ ``conditioned'' on running the policy $\pi$ in the POMDP. 

Similarly, as in MDPs, we aim to find a policy $\pi$ that maximizes~\eqref{eq:moe}, but we are actually interested in achieving a good performance on~\eqref{eq:mse}. It is easy to see how the value of~\eqref{eq:moe} can depart significantly from the true objective~\eqref{eq:mse}. Take, for instance, an observation matrix that maps every state to the same observation $\Obs (\bar x | s) = 1 \ \forall s \in \S$. It is clear that every policy is optimal for MOE in this setting, but the entropy on the true states can be arbitrarily bad. While those extreme cases are rather unrealistic, the observation matrix can be truly messed up in practice. We want to understand what are the settings that are worth addressing with MOE and what kind of guarantees we can get. In the following section, we provide answers to these questions by deriving theoretical bounds on the approximation of MSE with MOE that depends on crucial properties of the observation matrix.
\section{A Formal Characterization of Maximum Observation Entropy}
\label{sec:characterization}

In this section, we aim to characterize the gap $H(S|\pi) - H (X|\pi)$ induced by a chosen policy $\pi$, e.g., the policy that maximizes the MOE objective~\eqref{eq:moe}.
Due to the POMDP nature, in which only partial information (if any) on the true states is leaked to the agent, we cannot provide any general guarantee on the latter gap, which can be as large as
\begin{equation}
\label{eq:instance_independent_gap}
    |H(S|\pi)- H(X|\pi)| \leq \max \{ \log |\S|, \log |\X| \}.
\end{equation}
Nonetheless, we can provide \emph{instance-dependent} results that formally characterize the gap according to notable properties of the observation function in the given instance. First, we prove the following.
\begin{theorem}[Spectral Approximation Bounds]
\label{thr:spectral_bounds}
    Let $\mathbb{M}$ a POMDP and let $\pi \in \Pi \subseteq \Delta_{\X}^{\A}$ a policy. Then, it holds
    \begin{equation*}
        \log \left( \frac{1}{\sigma_{\max} (\Obs^{\circ -1})} \right)
        \leq 
        H (S | \pi) - H (X| \pi)
        \leq 
        \log ( \sigma_{\max} (\Obs) ).
    \end{equation*}
    % and
    % \begin{equation*}
    %     H(X | \pi) \geq \frac{ H (S | \pi) }{ \|p_S^\pi\|_\infty} + \frac{\|p_S^\pi\|_\infty - 1}{\|p_S^\pi\|_\infty} \log \left( \frac{|\S| - 1}{1 - \|p_S^\pi\|_\infty} \right) + \log \left( \frac{1}{\sigma_{\max} (\Obs)} \right).
    % \end{equation*}
\end{theorem}
\begin{proof}
    First, we derive the upper bound. Starting from $H(X|\pi)$, we have
    \begin{align}
        H (X | \pi) \geq H_2 &(X | \pi) = \log \left(\frac{1}{\norm{p^\pi_X}_2}\right) = \log \left( \frac{1}{\norm{\Obs \cdot p^\pi_S}_2} \right) \label{eq:spectral_1} \\
        &\geq \log \left(\frac{1}{\norm{\Obs}_2\norm{p_S^\pi}_2} \right) = \log \left(\frac{1}{\norm{p_S^\pi}_2} \right) + \log \left(\frac{1}{\norm{\Obs}_2} \right) =H_2(S | \pi) - \log \left(\sigma_{\max}(\Obs)\right)
    \end{align}
    where the first inequality comes from $H (V) \geq H_2 (V)$ for every variable $V$ and the second inequality from $\norm{\V \cdot v}_2 \leq \norm{\V}_2\norm{v}_2$ for every matrix $\V$ and vector $v$. Then, starting from $H(S|\pi)$, we get
    \begin{align}
        H (S|\pi) = \|p_S^\pi\|_\infty  \log \left(\frac{1}{\|p_S^\pi\|_\infty }\right) &+ \sum_{s: p_S^\pi (s) < \| p_S^\pi \|_\infty} p_S^\pi (s) \log\left(\frac{1}{p_S^\pi (s)}\right) \label{eq:spectral_2} \\
        &\qquad\qquad\leq \|p_S^\pi\|_\infty  H_\infty(S| \pi) + (1 - \|p_S^\pi\|_\infty )\log\left(\frac{|\S|-1}{1-\|p_S^\pi\|_\infty }\right)
    \end{align}
    where the inequality is obtained by letting $p^\pi_S$ be uniformly distributed outside of the entry $\|p_S^\pi\|_\infty$.
    By noting $H_\infty (V) \leq H_2 (V)$ and plugging \eqref{eq:spectral_2} back to \eqref{eq:spectral_1} we get
    \begin{equation}
        H(X | \pi) \geq \frac{ H (S | \pi) }{ \|p_S^\pi\|_\infty} + \frac{\|p_S^\pi\|_\infty - 1}{\|p_S^\pi\|_\infty} \log \left( \frac{|\S| - 1}{1 - \|p_S^\pi\|_\infty} \right) + \log \left( \frac{1}{\sigma_{\max} (\Obs)} \right) \label{eq:tight_spectral_upper_bound}
    \end{equation}
    which gives the result for $\| p^\pi_S \|_\infty \to 1$.\footnote{Note that~\eqref{eq:tight_spectral_upper_bound} is a tighter version of the upper bound than the one provided in the theorem statement, although it directly depends on the state distribution $p^\pi_S$ beyond spectral properties of $\Obs$. }

    To derive the lower bound, we proceed as follows. We start from the $H(X|\pi)$ definition to write
    \begin{align}
        H & (X|\pi) = \sum_{x \in \X} p^\pi_X (x) \log \left(\frac{1}{p^\pi_X (x)} \right)
        = \sum_{x \in \X} p^\pi_X(x) \log \left(\frac{\sum_{s \in \S} p^\pi_S(s)}{\sum_s p^\pi_S(s)\Obs(x|s)}\right)\sum_{s \in \S} p^\pi_S(s) \label{eq:spectral_3} \\
        &\leq \sum_{x \in \X} p^\pi_X(x) \sum_{s \in \S} p^\pi_S (s) \log \left(\frac{ p^\pi_S(s)}{ p^\pi_S(s)\Obs(x|s)} \right)
        %\label{eq:spectral_4}\\
        = H(S|\pi) + \sum_{x \in \X} p^\pi_X(x) \sum_{s \in \S} p^\pi_S(s) \log \left(\frac{ p^\pi_S(s)}{\Obs(x|s)}\right) \label{eq:spectral_5} \\
        &\leq H(S | \pi) +  \EV_{x \sim p_X^\pi} \EV_{s \sim p_S^\pi} \left[ \log ( \Obs^{\circ -1} (x | s) ) \right] %\label{eq:spectral_6} \\
        \leq H(S | \pi) + \log \left( \max_{x \in \X} \max_{s \in \S} \Obs^{\circ -1} (x | s) \right) \label{eq:spectral_7} \\
        &\leq H(S | \pi) + \log \left( \sigma_{\max} ( \Obs^{\circ -1} ) \right) \label{eq:spectral_8}
    \end{align}
    where we exploit $p_X^\pi (x) = \sum_{s \in \S} p_S^\pi (s) \Obs (x | s)$ and $\sum_{s \in \S} p_S^\pi (s) = 1$ to write~\eqref{eq:spectral_3}, we first apply the log-sum inequality and we split the logarithm to get~\eqref{eq:spectral_5}. Then, in~\eqref{eq:spectral_7}, we write the first inequality through the definition of the Hadamard inverse of $\Obs$ and noting that $p^\pi_S (s) \leq 1 \ \forall s \in \S$, we get the second inequality from $\EV [V] \leq \max (V)$ for any random variable $V$ and the monotonicity of the logarithm. Finally, we obtain the result~\eqref{eq:spectral_8} by $\| \mathbb{V} \|_\infty \leq \| \mathbb{V} \|_2 = \sigma_{\max} (\mathbb{V})$ for any matrix $\mathbb{V}$.
\end{proof}

Theorem~\ref{thr:spectral_bounds} gives bounds on the approximation gap that can be much tighter than the worst-case gap in~\eqref{eq:instance_independent_gap}. The bounds relate the gap to the scale of the transformation induced by the observation matrix on the distribution of the latent states, which is captured by the maximum singular value of $\Obs$ and $\Obs^{\circ -1}$, respectively. Notably, both sides of the bounds collapse to zero when the observation matrix is an identity matrix, i.e., when the states are fully observed. 

The bounds in Theorem~\ref{thr:spectral_bounds} only focus on spectral properties of the observation matrix $\Obs$. In a similar vein, we can provide an analogous characterization based on information properties of $\Obs$.

\begin{theorem}[Information Approximation Bound]
\label{thr:information_bound}
    Let $\mathbb{M}$ a POMDP, let $\pi \in \Pi \subseteq \Delta_{\X}^{\A}$ a policy, and let $H(X|S, \pi) = \EV_{s \sim p^\pi_S} [H(\Obs(\cdot|s))]$. Then, it holds
    \begin{equation*}
        H (S | \pi) \geq H (X| \pi) - H (X | S, \pi).
    \end{equation*}
\end{theorem}
\begin{proof}
    Starting from $H(X | \pi)$ definition, we can write
    \begin{align}
        H (X | \pi)  &= \sum_{x \in \X} p^\pi_X(x) \log \frac{1}{p^\pi_X(x)} 
        = \sum_{x \in \X} \sum_{s \in \S}\Obs(x|s)p^\pi_S(s)\log \frac{1}{\sum_{s' \in \S}\Obs(x|s')p^\pi_S(s')} \label{eq:information_2} \\
        &\leq \sum_{x \in \X}\sum_{s \in \S} \Obs(x|s) p^\pi_\S(s)\log \frac{1}{\Obs(x|s)p^\pi_S(s)} \label{eq:information_3} \\
        &= \sum_{x \in \X}\sum_{s \in \S} \Obs(x|s) p^\pi_\S(s)\log \frac{1}{p^\pi_S(s)}+ \sum_{x \in \X}\sum_{s \in \S} \Obs(x|s) p^\pi_S(s)\log \frac{1}{\Obs(x|s)} \label{eq:information_4} \\
        &= H (S | \pi) + \sum_{s \in \S} p^\pi_S (s) H (\Obs (\cdot | s)) = H (S | \pi) + H(X|S,\pi) \label{eq:information_5}
    \end{align}
    where we get~\eqref{eq:information_3} by noting $\sum_{s'\in \S}\Obs(x|s')p^\pi_S(s') \geq  \Obs(x|s)p^\pi_S(s)$, we split the logarithm to write~\eqref{eq:information_4}, we let $\sum_{x \in \X} \Obs (x | s) = 1$ and $\sum_{s \in \S} p_S^\pi (s) H (\Obs (\cdot | s)) = H(X|S, \pi)$ to obtain the result in~\eqref{eq:information_5}.
\end{proof}

Theorem~\ref{thr:information_bound} essentially states that the gap between the entropy on observations and true states is small as long as the policy $\pi$ induces visits to states where the observation function has low entropy, which is captured by the term $H(X|S,\pi) = \EV_{s \sim p_S^\pi}[H(\Obs (\cdot|s))]$. Just as Theorem~\ref{thr:spectral_bounds}, also the latter bound is tight when the true states are fully observed, collapsing the gap to zero.

The combination of Theorems~\ref{thr:spectral_bounds},~\ref{thr:information_bound} yield a nice description of the instances that is reasonable to address with a MOE approach, i.e., those for which the gap between the resulting policy and the optimal MSE policy is small thanks to the properties of the observation matrix. Unfortunately, policies in POMDPs have control over neither the spectral properties of the observation function nor whether the visited states have low-entropy observation distributions. In other words, while being descriptive, these results do not provide any further tool to actively address MSE in POMDPs.
In the next section, we reformulate the bound in Theorem~\ref{thr:information_bound} around quantities that can be actively controlled by a policy conditioned on observations and we provide a family of policy gradient algorithms to learn a MOE policy in those relevant instances.

Before diving into algorithmic solutions, it is interesting to confront the properties making a state entropy maximization problem on POMDPs easy and analogous requirements for RL in POMDPs. In the latter setting, we generally ask for an observation function that leaks significant information on the latent state. For instance, this is captured by a lower bound on the minimum singular value of $\Obs$ in the \emph{revealing} POMDP assumption~\citep[e.g.,][]{liu2022partially}. Instead, in state entropy maximization, we care less about identifying the latent state, and we can just focus on observations as long as $\Obs$ does not dramatically jeopardize the underlying state distribution.

%%%%%%%%%%%%%%%%%%%%%%%%%%%%%%%%%%%%%%%%%%%%%%%%%%%%%%%%%%%%%%%%
%% Section: MOE Properties (KNOWN MODEL)
%%%%%%%%%%%%%%%%%%%%%%%%%%%%%%%%%%%%%%%%%%%%%%%%%%%%%%%%%%%%%%%%
\section{Towards Principled Policy Gradients for MOE}
\label{sec:algorithm}

In the previous section, we analyzed the theoretical guarantees we get on the state entropy maximization problem by optimizing the MOE objective~\eqref{eq:moe}, but we did not yet describe how the latter optimization can be performed. Here we propose a family of Policy Gradient algorithms~\citep{peters2008reinforcement} to learn a MOE policy from sampled interactions with the POMDP.

First, we define a space of \emph{parametric} policies $ \pi_\theta \in \Pi_\Theta \subseteq \Pi$ where $\theta \in \Theta \subseteq \mathbb{R}^{|\X||\A|}$ are differentiable policy parameters.\footnote{See~\citet[][Section 1.3]{deisenroth2013survey} for common choices of parametric policy spaces.} The expression of the MOE objective does not allow for an easy computation of policy gradients. However, if we let $H(X|\vx) = -\sum_{x \in \X} \hat p_X (x|\textbf{x}) \log \hat p_X (x|\textbf{x})$ the observation entropy induced by a sequence of observations $\vx$, we can write a convenient trajectory-based counterpart of~\eqref{eq:moe}, namely
\begin{equation}
\label{eq:moe_trajectory}
    \max_{\pi_\theta \in \Pi_\Theta} \ \Big\{ \vH (X | \pi_\theta) :=
    \sum\nolimits_{ ( \vx, \va ) \in \X^T \times \A^T} q^{\pi_\theta}_{XA} (\vx, \va)H(X|\vx)\Big\}.
\end{equation}
Notably, the trajectory-based objective~\eqref{eq:moe_trajectory} is a lower bound to the MOE objective~\eqref{eq:moe}, due to the concavity of the entropy function and the Jensen's inequality~\citep[see][]{mutti2022challenging}. Thus, optimizing for~\eqref{eq:moe_trajectory} guarantees a non-degradation of our initial objective function~\eqref{eq:moe}, while it allows for an easy derivation of the gradient $\nabla_\theta$ w.r.t. the policy parameters.\footnote{The proof can be found in Appendix~\ref{apx:theory}.}
\begin{proposition}[Policy Gradient for MOE]
\label{prop:policy_gradient}
    Let $\pi_\theta \in \Pi_\Theta$ a parametric policy and let the \emph{policy scores} \emph{$\nabla_\theta \log \pi_\theta (\vx, \va) = \sum_{t \in [T]} \nabla_\theta \log \pi_\theta(\va[t] |\vx[t])$}. We can compute the \emph{policy gradient} of $\pi_\theta$ as
    \begin{equation}
        \nabla_\theta  \mathbf{H} (X | \pi_\theta) = \EV_{(\vx, \va) \sim q^\pi_{XA}} \Big[\nabla_\theta \log \pi_\theta (\vx,\va)H(X|\vx)\Big].
    \end{equation}
\end{proposition}
With the latter result, we can design a policy gradient algorithm based on REINFORCE~\citep{williams1992simple}. The procedure, described in Algorithm~\ref{alg:pg_pomdp}, initializes the policy parameters and then performs several iterations of gradient ascent updates.  As we shall see in the next section, Algorithm~\ref{alg:pg_pomdp} can be a simple yet effective solution to MOE optimization in various settings. However, the resulting policy can be underwhelming in domains where the observation matrix is particularly challenging. While we cannot overcome the barriers established in Theorems~\ref{thr:spectral_bounds},~\ref{thr:information_bound}, we can still exploit additional information on the observation function to further improve the performance.

\begin{algorithm}[t]
\caption{PG for MOE ({\color{purple} \textbf{Reg-MOE}})} \label{alg:pg_pomdp}
\begin{algorithmic}[1]
    \STATE \textbf{Input}: learning rate $\alpha$, number of iterations $K$, batch size $N$
    \STATE Initialize the policy parameters $\theta_1$
    \FOR{$k = 1, \ldots, K$}
        \STATE Sample $N$ trajectories $\{(\textbf{x}_i,\textbf{a}_i)\}_{i \in [N]}$ with the policy $\pi_{\theta_k}$
        \STATE Compute $\{H(X|\textbf{x}_i)\}_{i \in [N]}$ and  $\{ \nabla_\theta \log \pi_\theta (\textbf{x}_i,\textbf{a}_i) = \sum_{t\in [T]}\nabla_\theta \log \pi_\theta(\textbf{a}_i[t] | \textbf{x}_i[t])\}_{i \in [N]}$
        \STATE Update the policy parameters in the gradient direction \\
        {$ \qquad \ \theta_{k+1}  \leftarrow \theta_k + \alpha \frac{1}{N}\sum_i^N \nabla_\theta \log \pi_\theta (\textbf{x}_i,\textbf{a}_i ) \big(H(X|\textbf{x}_i)\textcolor{purple}{  - \beta\sum_{x \in \X} p_X (x|\textbf{x}_i) H (\Obs (x | \cdot)}\big)$}
    \ENDFOR
    \STATE \textbf{Output}: the final policy $\pi_{\theta_K}$
\end{algorithmic}
\end{algorithm}

\paragraph{Known Observation Matrix.}
With the knowledge of $\Obs$, we are tempted to directly optimize the lower bound to $H(S|\pi)$ provided in Theorem~\ref{thr:information_bound} by trading-off high entropy on observations ($H(X|\pi)$) with the entropy of their emission ($H(X|S, \pi)$). Unfortunately, we do not have access to the state distribution $d^\pi_S$ to compute the expectation $H(X|S, \pi) = \EV_{s \sim p^\pi_S} [H(\Obs(\cdot|s))]$. Nonetheless, we can rework the lower bound into an alternative form where all of the terms are known and can be controlled by a policy conditioned on observations only, as it demonstrates the following corollary to Theorem~\ref{thr:information_bound}.
\begin{corollary}[Actionable Lower Bound]\label{th:practical_bound}
    Let $\mathbb{M}$ a POMDP, let $\pi \in \Pi \subseteq \Delta_{\X}^{\A}$ a policy, and let $H(S|X, \pi) = \EV_{x \sim p^\pi_X} [H(\Obs(x | \cdot))]$. Then, it holds
    \begin{equation}  
         H (S | \pi) \geq  H (X | \pi) - H(S|X,\pi) + \log(\sigma_{\max} (\Obs)).
    \end{equation}
\end{corollary}
\begin{proof}
    The result follows straightforwardly through further manipulation of Theorem~\ref{thr:information_bound}. We have,
    \begin{align}
         H (S | \pi) \geq  H (X | \pi) - H(X|S,\pi) &= H (X | \pi) - H(S|X,\pi) + H(S|\pi) - H(X|\pi) \label{eq:regularization_1}\\
        %&=  H (X | \pi) -  \sum_{x \in \X} p^\pi_X (x) H (\Obs (x | \cdot)) - [  H (X | \pi) - H (S | \pi)] \label{eq:regularization_2}\\
        &\geq  H (X | \pi) -  \sum_{x \in \X} p^\pi_X (x) H (\Obs (x | \cdot)) +\log(\sigma_\text{max}(\Obs)) \label{eq:regularization_3}
    \end{align}
    where~\eqref{eq:regularization_1} is the result of the application of the Bayes rule to the conditional entropy $ H(X|S,\pi)$ and~\eqref{eq:regularization_3} follows from the fact that $ H (X | \pi) - H (S | \pi) \geq  -\log(\sigma_\text{max}(\Obs))$ due to Theorem~\ref{thr:spectral_bounds}.
\end{proof}
From the latter result, we get a lower bound to $H(S|\pi)$ that can be controlled, as we flipped the conditioning from $H(X|S,\pi)$ to $H(S|X, \pi)$, which we can compute by taking an expectation with the observation distribution. For every $\beta \in (0,1)$, we can write a regularized version of~\eqref{eq:moe_trajectory} as
\begin{equation}
    \vH\nolimits_\beta (X | \pi) := \sum\nolimits_{(\vx, \va) \in \X^T \times \A^T} q^\pi_{XA} (\vx,\va)\Big( H(X|\vx) - \beta\sum\nolimits_{x \in \X} \hat p_X (x|\vx) H (\Obs (x | \cdot)\Big),
\end{equation}
which we call \emph{Regularized MOE} (Reg-MOE), and a slight variation of the Algorithm~\ref{alg:pg_pomdp} (highlighted in the pseudocode) to optimize the regularized objective. In the next section, we provide an empirical validation of the proposed PG algorithms to describe their respective strengths and weaknesses. Note that the presented algorithms can be further enhanced with the same technical solutions of advanced policy optimization algorithms for the MSE objective~\citep[e.g.,][]{mutti2021task,liu2021behavior,seo2021state,yarats2021reinforcement} to address continuous and high-dimensional domains.

\section{Numerical Validation}
\label{sec:experiments}

Here we provide a brief numerical validation of the theoretical results provided in Section~\ref{sec:characterization} and the algorithmic solutions proposed in Section~\ref{sec:algorithm}. Especially, we aim to show that
\begin{itemize}[itemsep=0pt, leftmargin=16pt, topsep=0pt]
    \item[(a)] Optimizing MOE is particularly effective when the observation matrix is ``well-behaved'';
    \item[(b)] Optimizing MOE is bound to fail when the observation matrix is not ``well-behaved'';
    \item[(c)] Additional knowledge of the observation structure can be sometimes exploited to improve the performance in the latter challenging cases by optimizing the regularized MOE.
\end{itemize}
Intuitively, an observation matrix is ``well-behaved'' when it does not induce a significant transformation of the state distribution, keeping the approximation gap between MOE and MSE small. Thanks to Theorems~\ref{thr:spectral_bounds},~\ref{thr:information_bound} we can provide a formal characterization of this property. In the experiments below, we measure the latter through the average entropy of the observation function $\EV[H(\Obs)] = \sum_{s \in \S} H(\Obs (\cdot|s)) / |\S| $ on the lines of the information bound in Theorem~\ref{thr:information_bound}.

In Figure~\ref{subfig:image1} we test (a) by showing that the performance of the algorithms accessing observations only, i.e., \emph{PG for MOE} and \emph{PG for Reg-MOE}, is remarkably close to the ideal baseline having access to the true states, i.e., \emph{PG for MSE}. This is due to the low average entropy of the observation function: Although the agent cannot know its exact position, maximizing the entropy of observations still leads to a large entropy over the latent states.

This is not the case in the experiment in Figure~\ref{subfig:image2}, where the gridworld configuration is the same, but the observation function is now more challenging, i.e., more entropic on average. The significant gap between the algorithms optimizing MOE and the ideal baseline is a testament of (b) and a corroboration of the theoretical limits of the MOE approach, which are formally provided in Theorems~\ref{thr:spectral_bounds},~\ref{thr:information_bound}. \emph{PG for MOE} and \emph{PG for Reg-MOE} can still successfully maximize the entropy over observations, but cannot avoid a significant mismatch with the resulting entropy over latent states.

However, not all the domains with challenging (i.e., entropic) observations are hopeless for the MOE approach, especially when we can exploit knowledge on how the observations are themselves generated. In Figure~\ref{subfig:image3}, we report a further experiment in which the observation matrix has a block with very high entropy (in which observations are almost random) and a block with nearly deterministic observations. \emph{PG for MOE} does not exploit the structure of $\Obs$ and cannot distinguish between observations that are \emph{reliable} from those that are not. Instead, the regularization term in \emph{PG for Reg-MOE} leads to more visitations of reliable observations (i.e., generated with lower entropy) effectively reducing the gap with the ideal baseline (\emph{PG for MSE}), which corroborates both (c) and the result in Corollary~\ref{th:practical_bound}.

As a bottom line, this numerical validation shows that the MOE approach, while not being a solution to every POMDP instance, can still provide a remarkable performance on domains where the observation matrix is not too challenging or when its knowledge can be exploited.

\begin{figure*}[t]
    \centering \includegraphics[width=0.5\textwidth]{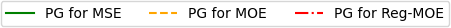}

    \begin{subfigure}[b]{0.3\textwidth}
        \includegraphics[width=\textwidth]{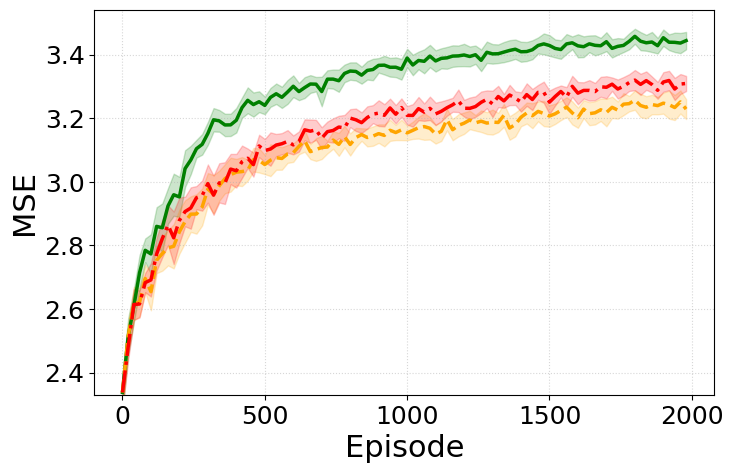}
        \vspace{-15pt}
        \caption{\centering Well-behaved observations $\EV [H (\Obs)] \approx 1$}
        \label{subfig:image1}
    \end{subfigure}
    \hfill
    \begin{subfigure}[b]{0.3\textwidth}
        \includegraphics[width=\textwidth]{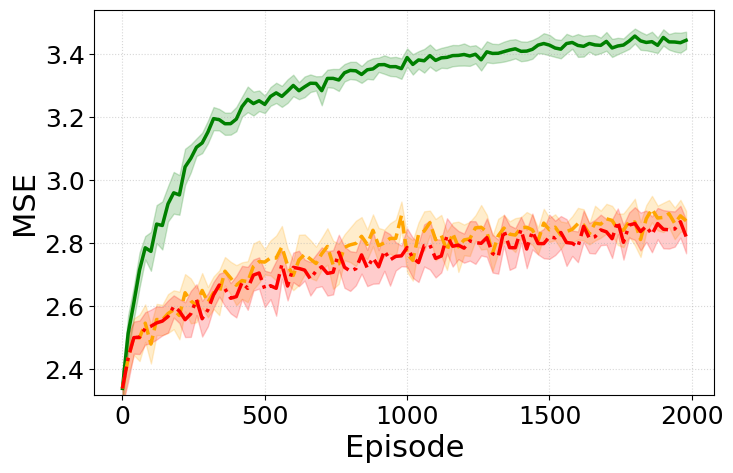}
        \vspace{-15pt}
        \caption{\centering Challenging observations $\EV [H(\Obs)] \approx 2.2$}
        \label{subfig:image2}
    \end{subfigure}
    \hfill
    \begin{subfigure}[b]{0.3\textwidth}
        \includegraphics[width=\textwidth]{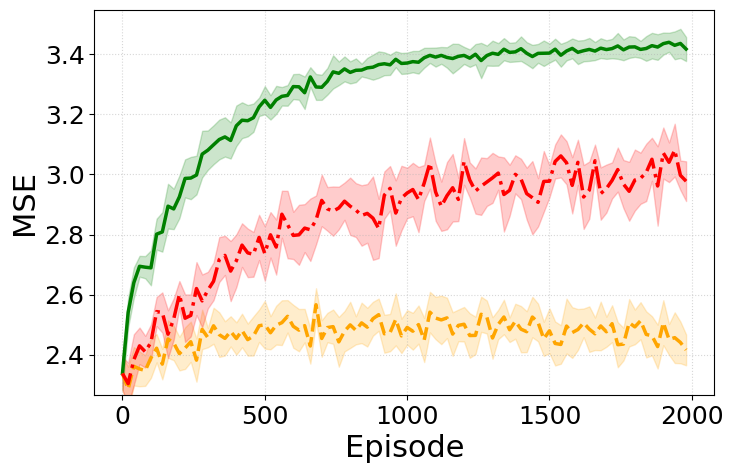}
        \vspace{-15pt}
        \caption{\centering Challenging observations with structure $\EV [H(\Obs)] \approx 1.85$}
        \label{subfig:image3}
    \end{subfigure}
    \vspace{-5pt}
    \caption{Entropy on latent states (MSE) achieved by \emph{PG for MSE}, \emph{PG for MOE}, and \emph{PG for Reg-MOE} in gridworlds with various $\Obs$. We report the average and 95\% c.i. over 16 runs.}
    \vspace{-0.5cm}
    \label{fig}
\end{figure*}

% As one would expect, MOE performs well when the observations are well-behaved, as the example in Fig.~\ref{subfig:image1}, while adding the regularization does not hurt. However, major gaps arise when considering harder settings, as it is the case for Fig.~\ref{subfig:image2}, and there is little one can do. On the other hand, when the agent is given the chance to exploit well-behaved emissions, as for the glasses scenario in Fig.~\ref{subfig:image3}, this structure is successfully exploited by the Reg-MOE objective, that is able to recover fairly good performances even under almost random emission matrices. The role of the regularization here is essential, as one can see in Fig.~\ref{subfig:image4}, since it is successfully minimized by Reg-MOE, while the MOE objective only is not able to control over it, leading to catastrophic performances.

%%%%%%%%%%%%%%%%%%%%%%%%%%%%%%%%%%%%%%%%%%%%%%%%%%%%%%%%%%%%%%%%
%% Section: RW
%%%%%%%%%%%%%%%%%%%%%%%%%%%%%%%%%%%%%%%%%%%%%%%%%%%%%%%%%%%%%%%%
\section{Related Work}
\label{sec:related_work}

This work rests in the intersection between POMDPs, state entropy maximization, and policy optimization. Here we report a list of the most relevant contributions in those areas.

\textbf{POMDPs.}~~
Learning and planning problems in POMDPs have been extensively studied. In the most general formulation, POMDPs are known to be computationally and statistically intractable~\citep{papadimitriou1987complexity, krishnamurthy2016pac, jin2020sample}. Nonetheless, several recent works have analyzed tractable sub-classes of POMDPs under convenient structural assumptions, such as~\citep[][]{jin2020sample, golowich2022planning, chen2022partially, liu2022partially, zhan2022pac, zhong2023gec}.

\textbf{State Entropy Maximization.}~~
State entropy maximization in MDPs has been introduced in~\citet{hazan2019provably}, from which followed a variety of subsequent works focusing on the problem from various perspectives~\citep{lee2019smm, mutti2020ideal, mutti2021task, mutti2022importance, mutti2022unsupervised, mutti2023unsupervised, zhang2020exploration, guo2021geometric, liu2021behavior, liu2021aps, seo2021state, yarats2021reinforcement, nedergaard2022k, yang2023cem, tiapkin2023fast, jain2023maximum, kim2023accelerating, zisselman2023explore}.
Among them,~\citet{savas2022entropy} indeed study the problem of maximizing the entropy over trajectories induced in a POMDP, yet we are the first to formulate \emph{state} entropy maximization in POMDPs in this paper and, concurrently, \citet{zamboni2024explore}.

\textbf{Policy Optimization.}~~
First-order methods have been extensively employed to address non-concave policy optimization~\citep{sutton1999policy, peters2008reinforcement}. In this work, we proposed a \emph{vanilla} policy gradient estimator~\citep{williams1992simple} as a first step, yet several further refinements could be made, such as natural gradient~\citep{kakade2001natural}, trust-region schemes~\citep{schulman2015trust}, and importance sampling~\citep{metelli2018policy}.

%%%%%%%%%%%%%%%%%%%%%%%%%%%%%%%%%%%%%%%%%%%%%%%%%%%%%%%%%%%%%%%%
%% Section: CONCLUSIONS
%%%%%%%%%%%%%%%%%%%%%%%%%%%%%%%%%%%%%%%%%%%%%%%%%%%%%%%%%%%%%%%%
\section{Conclusions}
\label{sec:conclusions}

In this paper, we made a step forward into generalizing state entropy maximization in POMDPs.
Specifically, we addressed the problem of learning a policy conditioned only by observations that target the entropy over the latent states. We proposed the simple approach of optimizing the entropy over observations in place of latent states and we formally characterized the instances where it is effective by deriving approximation bounds of the latent objective that depend on the structure of the observation matrix. 
Finally, we design a family of policy gradient algorithms to optimize the observation entropy in practice and to exploit knowledge of the observation structure when available.

Before concluding, it is worth mentioning that state entropy maximization can find further motivation in POMDPs beyond its common use in MDP settings. While how those methods can benefit offline data collection and transition model estimation is less obvious under partial observability, it is worth noting that the reward in a POMDP is usually defined over the true states, such that pre-training a policy to explore over them is still relevant~\citep{eysenbach2021information}. Moreover, the policy we aim to learn is commonly a mapping from beliefs to actions, for which ensuring good exploration over beliefs is important. Interestingly, in POMDPs one may choose to pre-train belief representations alone, to be transferred to various downstream tasks. For this problem, accessing data with coverage over the space of beliefs and, consequently, true states is essential. Exploring other potential uses of state entropy maximization in POMDPs is a nice future direction.

Future works may extend our results in many other directions, such as enhancing our algorithms by incorporating recent advancements in policy optimization for MSE~\cite[e.g.,][]{liu2021behavior} and designing alternative objectives and algorithms to target domains where the entropy of observations is not enough~\citep{zamboni2024explore}.
To conclude, we believe that this work sets a crucial first step in the direction of extending state entropy maximization to yet more practical settings.

\bibliography{biblio}
\bibliographystyle{rlc}

%%%%%%%%%%%%%%%%%%%%%%%%%%%%%%%%%%%%%%%%%%%%%%%%%%%%%%%%%%%%%%%%
%% Appendices
%%%%%%%%%%%%%%%%%%%%%%%%%%%%%%%%%%%%%%%%%%%%%%%%%%%%%%%%%%%%%%%%
\appendix

\newpage

\section{Missing Proofs}
\label{apx:theory}

Here we report the derivations of the policy gradient reported in Proposition~\ref{prop:policy_gradient}.
% For the trajectory-based MOE proxy $\textbf{H}(X|\pi_\theta)$ and a policy $\pi_\theta \in \Pi_\X$ parametrized by $\theta \in \Theta \subseteq \mathbb{R}^{|\X||\A|}$, it is possible to write:
Especially, we write
\begin{align}
    \nabla_\theta \textbf{H}(X|\pi_\theta) &=  \nabla_\theta \sum\nolimits_{(\vx, \va) \in \X^T \times \A^T} q^{\pi_\theta}_{XA} (\vx, \va)H(X|\vx)
    \\
    &=  \sum\nolimits_{(\vx, \va) \in \X^T \times \A^T} \Big( \nabla_\theta q^{\pi_\theta}_{XA} (\vx, \va)\Big)H(X|\vx) \\ 
    &=  \sum\nolimits_{(\vx, \va) \in \X^T \times \A^T} q^{\pi_\theta}_{XA} (\vx, \va) \nabla_\theta \log q^{\pi_\theta}_{XA} (\vx, \va) H(X|\vx) \\
    &=  \EV_{(\vx, \va) \sim q^{\pi_\theta}_{XA}}\big[\nabla_\theta \log q^{\pi_\theta}_{XA} (\vx, \va) H(X|\vx)\big]
\end{align}
by exploiting the linearity of the expectation to go from the first to the second equality, then applying the common log-trick~\citep{peters2008reinforcement} and finally recognising the sum as an expectation again.

To derive the gradient we then have to provide the calculation of the \emph{policy scores} $\nabla_\theta \log q^{\pi_\theta}_{XA} (\vx, \va)$. For every $\pi \in \Pi_\Theta$, we notice that $q^{\pi_\theta}_{XA} (\vx,\va) = \prod_{t \in [T]}Pr(x_t = \vx[t])\pi_\theta (a_t = \va[t] |x_t = \vx[t])$ and that the only term depending on $\theta$ is the policy itself. By exploiting the properties of the logarithm we have
\begin{align}
    \nabla_\theta \log q^{\pi_\theta}_{XA} (\vx, \va) &=  \sum_{t \in [T]} \nabla_\theta \log \pi_\theta (a_t = \textbf{a}[t] |x_t = \textbf{x}[t])
\end{align}

which leads to the standard REINFORCE formulation~\citep{williams1992simple}.

\clearpage

\section{Additional Details on the Experiments}
\label{apx:experiments}

In the following, we report additional details on the experiments of Section~\ref{sec:experiments}. Specifically, we describe the employed domains and their properties in Appendix~\ref{apx:environments_visualization}, we comment on the choice of hyper-parameters in Appendix~\ref{apx:hyper_parameters}, and on the effect of the regularization on the results of \emph{PG for Reg-MOE} in Appendix~\ref{apx:regularization}.

\subsection{Domains}
\label{apx:environments_visualization}

\begin{minipage}{\textwidth}
\begin{minipage}{0.5\textwidth}
\centering
\includegraphics[width=0.5\textwidth]{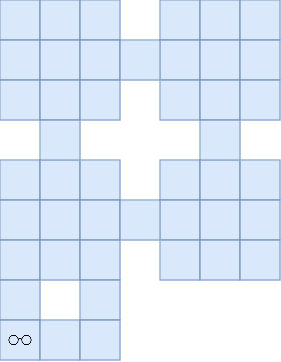}\label{fig:env}
\end{minipage}
\begin{minipage}{0.5\textwidth}

Most of the reported experiments refer to the gridworld reported on the left, which is composed of a set of rooms connected by narrow corridors. The grid is composed of $44$ cells, which define both the set of states ($|\S| = 44$) and observations ($|\X| = 44$). The set of actions $\A$ include an action to move to the adjacent cell in every direction ($|\A| = 4$). To every action is associated a probability of failure $\bar p = 0.1$ that leads the agent to an adjacent cell (at random) different from the one intended by the taken action. The episode horizon is $T = 55$ and the initial state distribution $\mu$ was set to be a deterministic over the top-left cell. 
\end{minipage}
\vspace{0.5cm}
\end{minipage}

\begin{figure*}[b!]
    \centering 
    \begin{subfigure}[b]{0.30\textwidth}
        \includegraphics[width=1\textwidth]{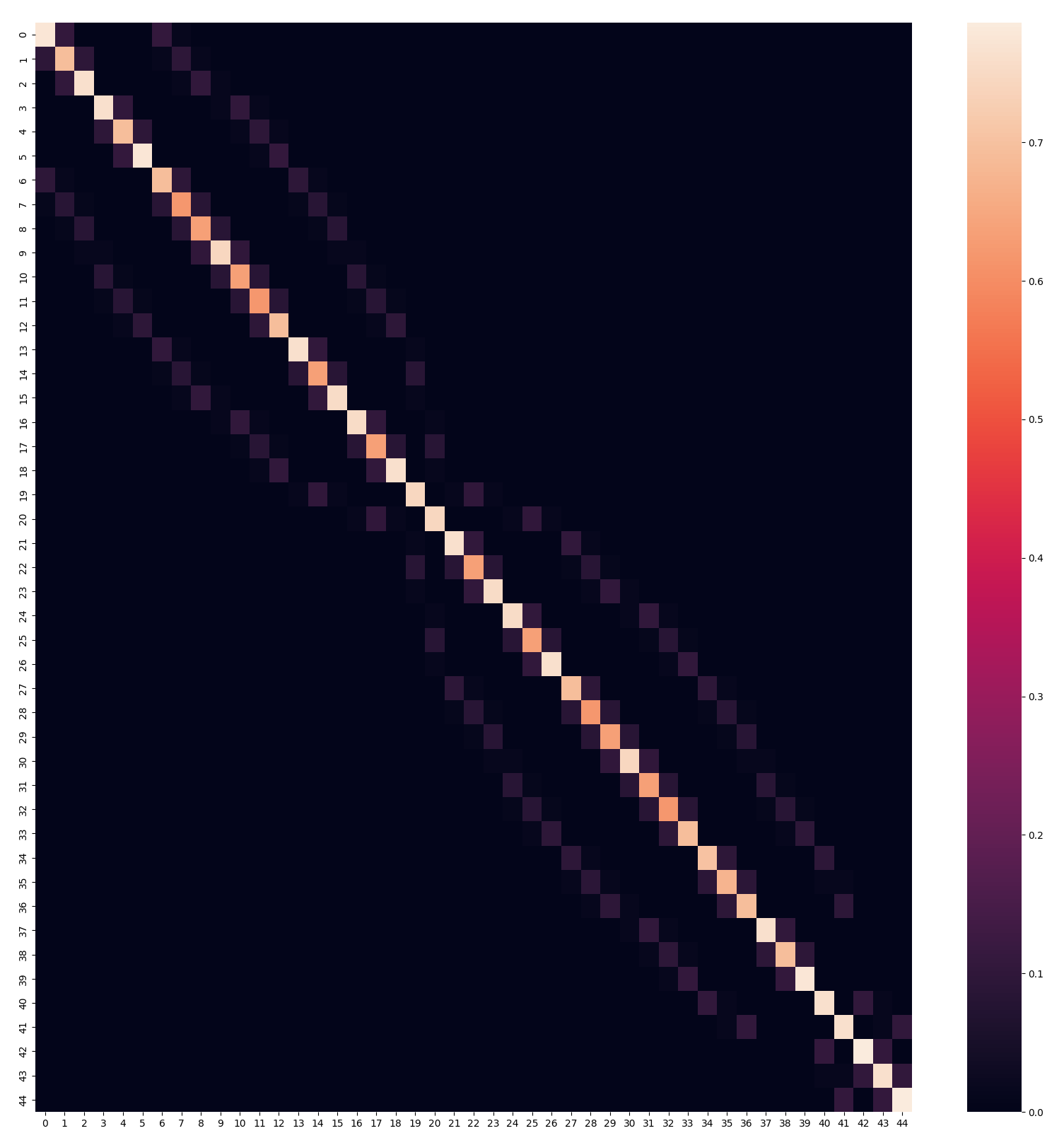}
        \caption{Observation matrix of the experiment in Figure \ref{subfig:image1}}
        \label{subfig:obs_image1}
    \end{subfigure}
    \hfill 
    \begin{subfigure}[b]{0.30\textwidth}
        \includegraphics[width=\textwidth]{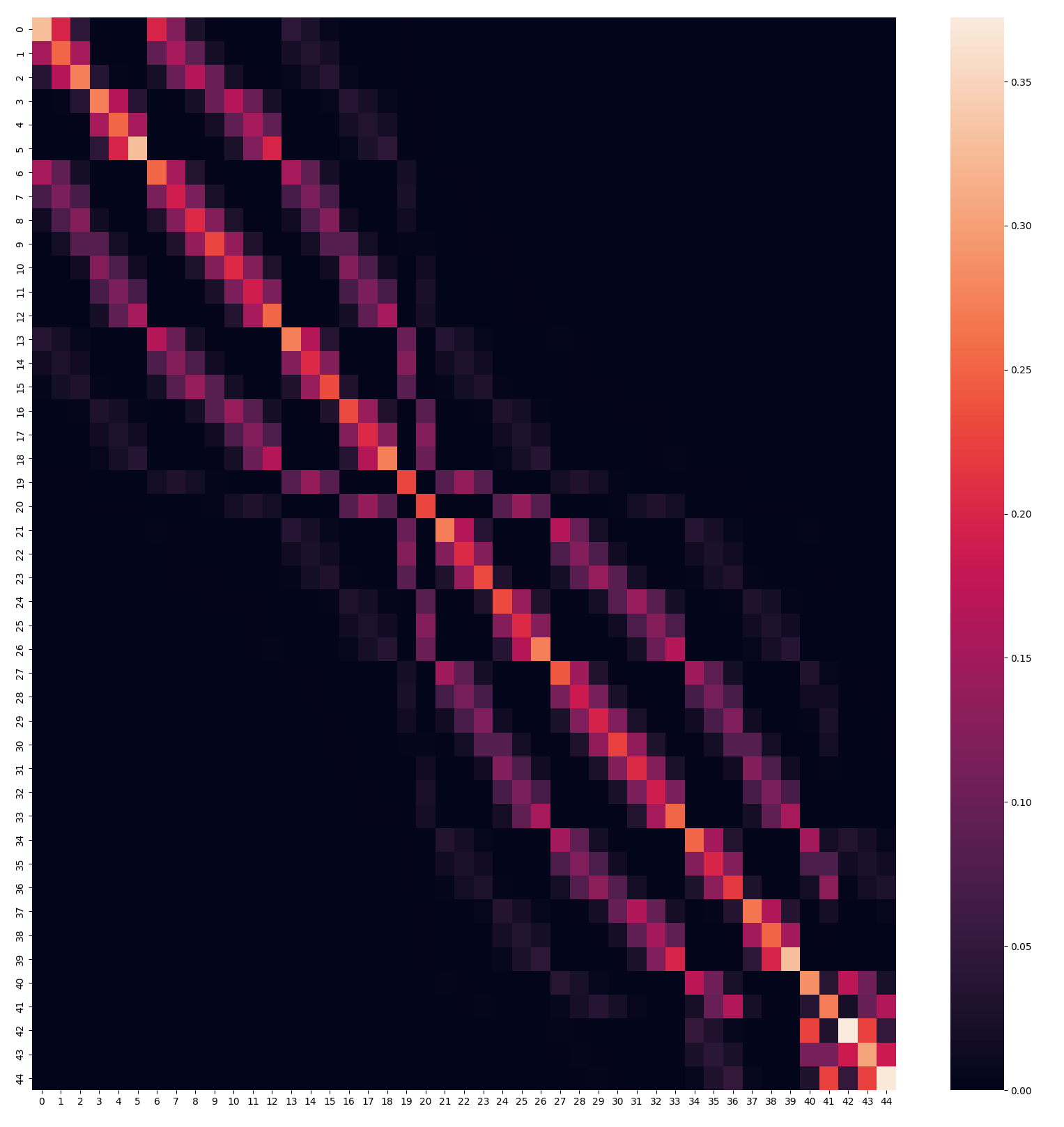}
        \caption{Observation matrix of the experiment in Figure \ref{subfig:image2}}
        \label{subfig:obs_image2}
    \end{subfigure}
    \hfill
    \begin{subfigure}[b]{0.30\textwidth}
        \includegraphics[width=\textwidth]{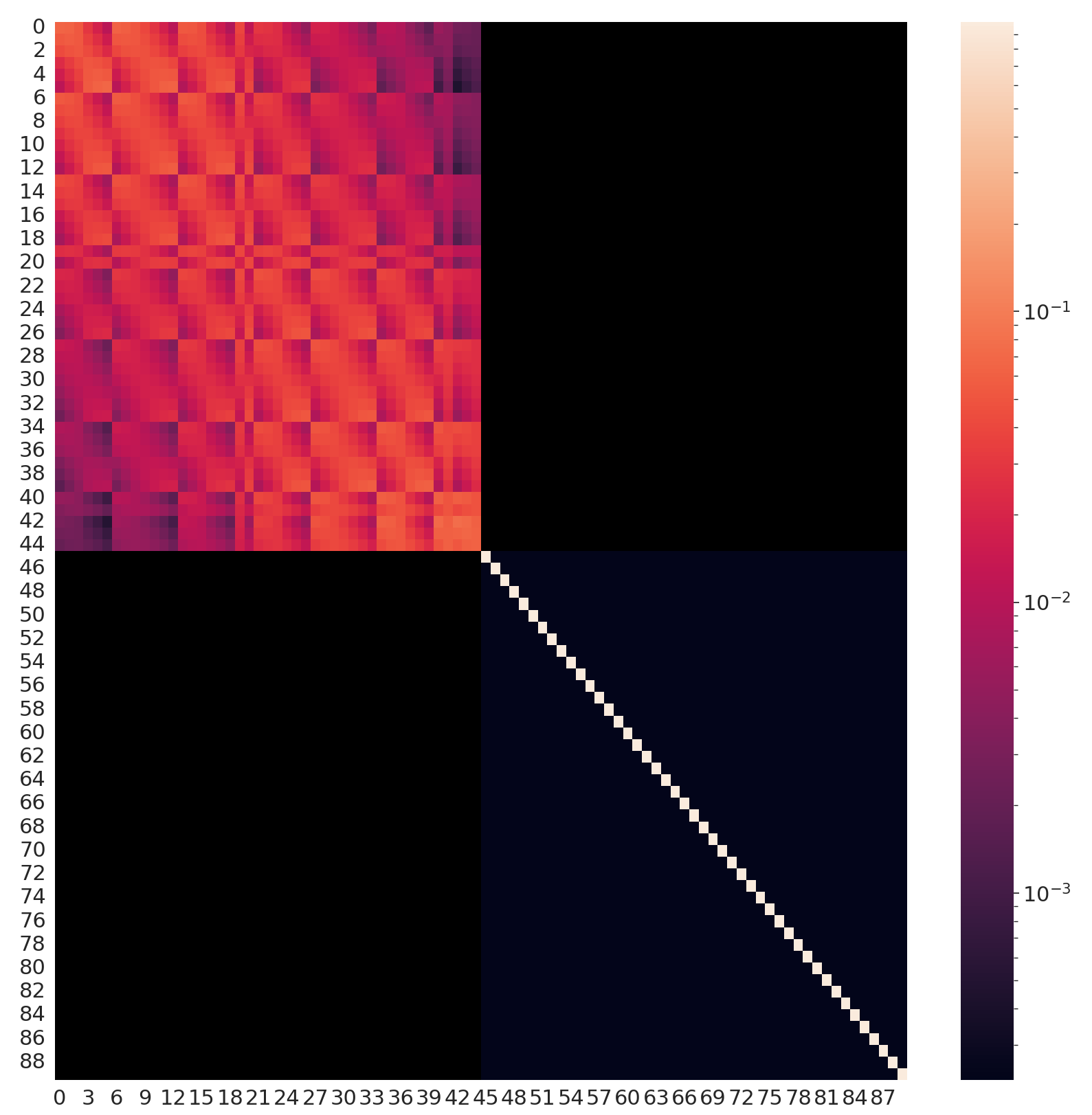}
        \caption{Observation matrix of the experiment in Figure \ref{subfig:image3}}
        \label{subfig:obs_image3}
    \end{subfigure}

    \caption{Heatmaps representing the observation matrix $\Obs$ employed in the experiments of Section~\ref{sec:experiments}. Note that in Figure~\ref{subfig:obs_image3} the colormap has logarithmic scale.}
    \label{fig:observation_matrices}
\end{figure*}

The glasses icon in the bottom left cell of the grid represents a state that ``flips'' the behavior of the observations. This is only relevant in the experiment in Figure~\ref{subfig:image3} and is better explained below. All the experiments of Section~\ref{sec:experiments} were performed with a regularization factor $\beta = 0.8$ (for \emph{PG for Reg-MOE}) and a learning rate of $\alpha=0.9$. Finally, the batch size was $N = 6$ and the number of independent runs was set to $16$.

\paragraph{Observations.}
The observations were set to be Gaussian distributions $\mathcal G(0,\sigma^2)$ over the Manhattan distance centered in the true state and without caring about any obstacles, with $0$ mean and different values of variance $\sigma^2$. The resulting observation matrices are reported in Figure~\ref{fig:observation_matrices}. Finally, the effect of ``wearing'' the glasses (i.e., reaching the bottom-left cell of the grid) is to make the observation function fully deterministic. Note that the information on whether the agent wears the glasses is encoded in the state themselves, doubling the size of the set of states to $|\S| = 88$. 

\subsection{Hyper-Parameters Selection}
\label{apx:hyper_parameters}

 In this section, we briefly discuss the choice behind the selection of specific hyper-parameters employed in the experiments. 

\paragraph{Learning Rate.}
As for the learning rate $\alpha$, a value of $\alpha = 0.9$ was selected across the experiments. As one can see from Figure~\ref{fig:alpha_val}, the best performance were reached with a learning rate between $\alpha=1$ and $\alpha=0.7$, so $\alpha = 0.9$ can be seen as a robust choice across the boards.

\begin{figure*}[ht]
    \centering \includegraphics[width=0.9\textwidth]{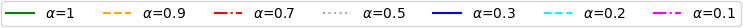}
    
    \begin{subfigure}[b]{0.32\textwidth}
        \includegraphics[width=\textwidth]{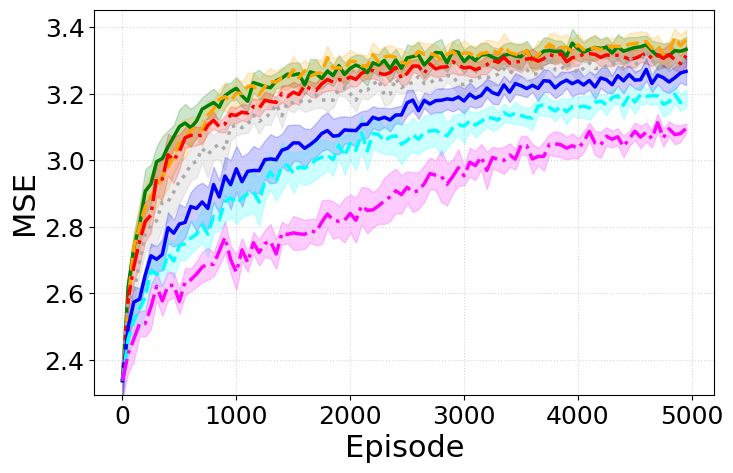}
        \caption{PG for MOE (observation matrix as in Figure~\ref{subfig:obs_image1})}
    \end{subfigure}
    \hfill
    \begin{subfigure}[b]{0.32\textwidth}
        \includegraphics[width=\textwidth]{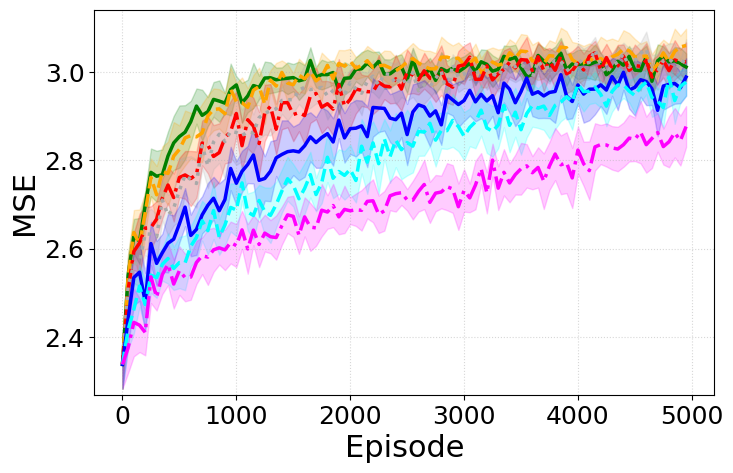}
        \caption{PG for MOE (observation matrix as in Figure~\ref{subfig:obs_image2})}
    \end{subfigure}
    \hfill
    \begin{subfigure}[b]{0.32\textwidth}
        \includegraphics[width=\textwidth]{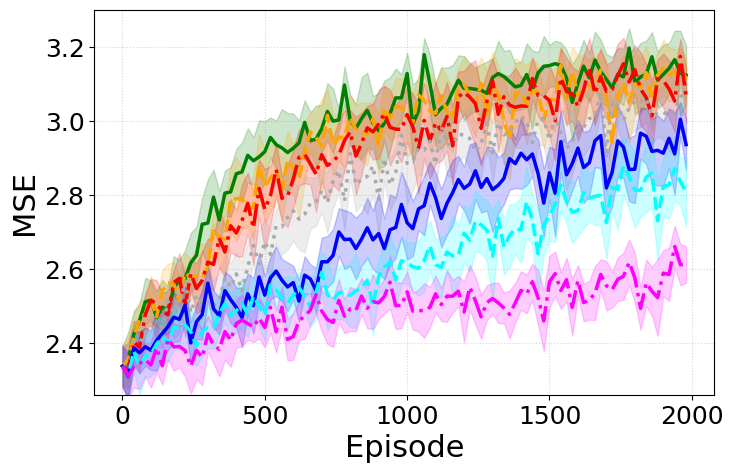}
        \caption{PG for Reg-MOE (observation matrix as in Figure~\ref{subfig:obs_image3})}
    \end{subfigure}
    %\hfill
    
    \caption{Comparison of the performance with different values of the learning rate for various algorithms and domains.}
    \label{fig:alpha_val}
\end{figure*}

\paragraph{Regularization.}
As for the regularization term $\beta$, the best performance for the various instances was generally reached with $\beta \in (0.3,1)$, as shown in Figure~\ref{fig:beta_val} (the learning rate is fixed to $\alpha = 0.9$). For lower values of $\beta$, the effect of the regularization is almost negligible, while for higher values of $\beta$ the agent tended to over-optimize the regularization term in place of the entropy over observations, reducing performance. As one would expect, the best value for the regularization depend on the specific POMDP instance.

\begin{figure*}[ht]

    \centering \includegraphics[width=0.9\textwidth]{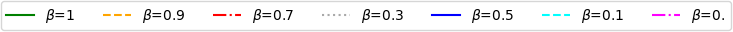}

    \begin{subfigure}[b]{0.32\textwidth}
        \includegraphics[width=\textwidth]{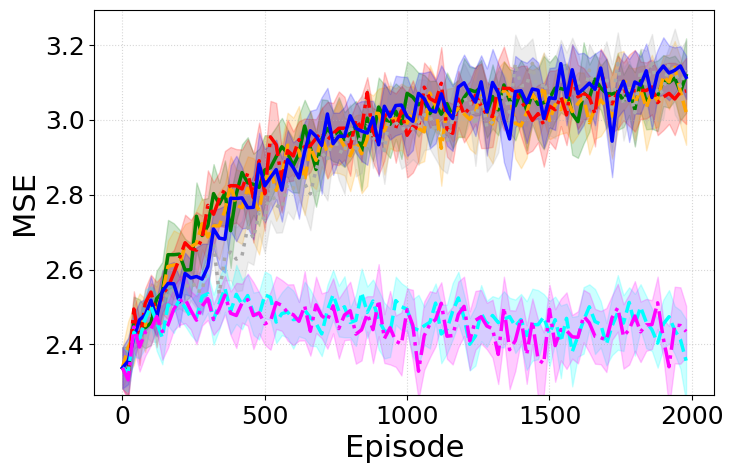}
        \caption{With glasses (variance of observations $\sigma^2 = 10$)}
    \end{subfigure}
    \hfill
    \begin{subfigure}[b]{0.32\textwidth}
        \includegraphics[width=\textwidth]{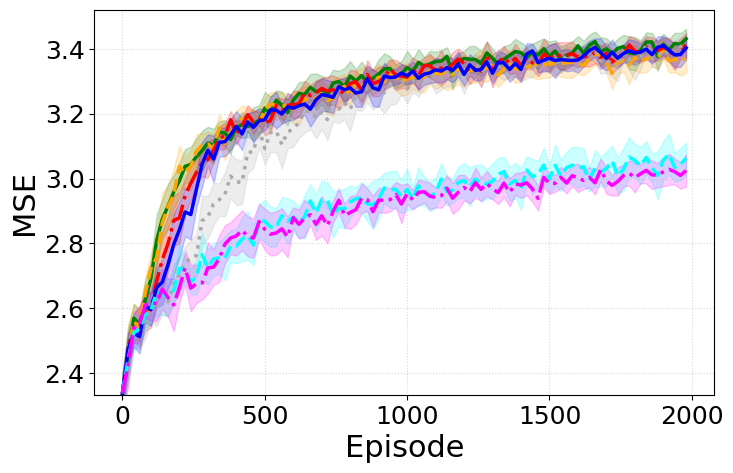}
        \caption{With glasses (variance of observations $\sigma^2 = 1$)}
    \end{subfigure}
    \hfill
    \begin{subfigure}[b]{0.32\textwidth}
        \includegraphics[width=\textwidth]{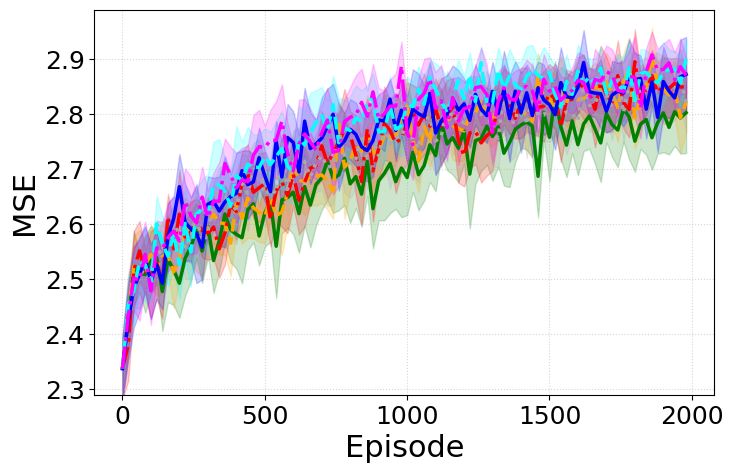}
        \caption{Without glasses (variance of observations $\sigma^2 = 0.25$)}
    \end{subfigure}
    %\hfill
    
    \caption{A comparison of different values of regularization for varying emission matrices' quality and settings with and without glasses. For the low value of regularization, the performances of Reg-MOE are equivalent to the MOE performances.}
    \label{fig:beta_val}
\end{figure*}

\subsection{Further Insights on the Effect of the Regularization}
\label{apx:regularization}

\begin{minipage}{\textwidth}

   \begin{minipage}{0.5\textwidth} In this section, we further investigate the effect of the regularization term. For this specific test, we consider a different gridworld configuration than previous experiments, which is reported on the right. The observation matrix is designed as a Gaussian $\mathcal G(0,\sigma^2)$ over the Manhattan distance in the blue rooms, while it is deterministic (and thus fully revealing) in the red room. 
   For this experiment, we set the variance to $\sigma^2=1$, the regularization term to $\beta = 0.3$, and the horizon $T = 40$. As for the remaining parameters, they are kept as in the previous experiments.
   \end{minipage}    
    \begin{minipage}{0.5\textwidth}\centering\includegraphics[width=0.6\textwidth]{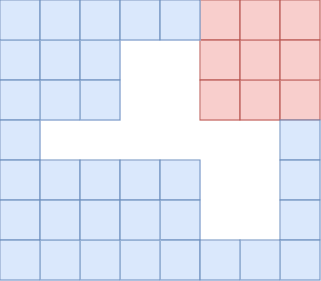}\end{minipage}
\end{minipage}

Figure~\ref{fig:policy} shows that, in this experiment, the two learned policies have similar performances. Yet, as can be seen, while the policy trained with \emph{PH for MOE} tries to explore the environment uniformly, the one trained with \emph{PG for Reg-MOE} successfully explored the portion of the grid with lower entropy in the observations, to later address a deeper exploration of the remaining rooms. This behaviour exactly aligns with the role of the regularization term, which should indeed make the agent prefer observations that are emitted with lower entropy by the observation function.

\begin{figure*}[h]

    \centering 
    \begin{subfigure}[b]{0.3\textwidth}
        \includegraphics[width=1\textwidth]{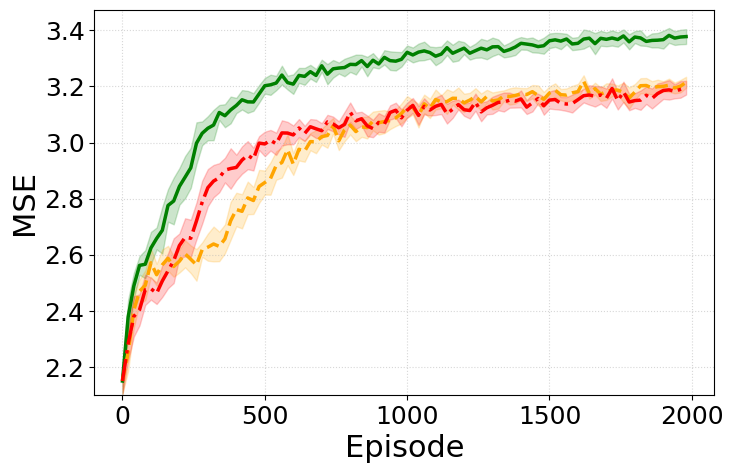}
        \caption{MSE performance of PG for MSE (green), PG for MOE (orange), PG for Reg-MOE (red)}
    \end{subfigure}
    \hfill
    \begin{subfigure}[b]{0.3\textwidth}
        \includegraphics[width=\textwidth]{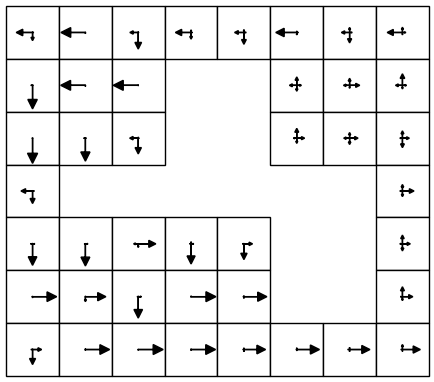}
        \caption{PG for MOE}
    \end{subfigure}
    %\hfill
    \begin{subfigure}[b]{0.3\textwidth}
        \includegraphics[width=\textwidth]{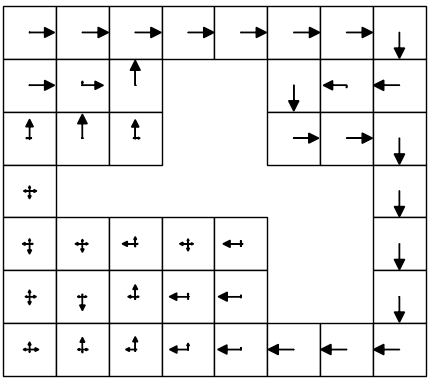}
        \caption{PG for Reg-MOE}
    \end{subfigure}
    %\hfill
    
    \caption{Comparison of the policies learned by PG for MOE and PG for Reg-MOE over 2000 episodes. The magnitude of each arrow is proportional to the probability of the policy to choose that action, after marginalizing over all the possible observations emitted in that state.}
    \label{fig:policy}
\end{figure*}

\end{document}